\documentclass[runningheads]{llncs}

\usepackage{amsmath,amssymb}

\usepackage{algorithm}
\usepackage[noend]{algorithmic}
\usepackage{xspace}

\spnewtheorem*{proofv}{Proof of Theorem 1}{\itshape}{\rmfamily}

\usepackage{graphicx}
\graphicspath{{.}{images/}}

\usepackage[misc]{ifsym}

\mathchardef\mhyphen="2D

\DeclareMathOperator*{\argmin}{arg\,min}

\newcommand{\bigO}{\mathcal{O}}

\newcommand{\np}{\ensuremath{\mathbf{NP}}\xspace}

\newcommand{\kkm}{{\tt Kernel $k$-means}\xspace}
\newcommand{\nnk}{{\ensuremath c}\xspace}
\newcommand{\nncost}{{\ensuremath \rm c\mhyphen NNC}\xspace}
\newcommand{\nnck}{{\ensuremath \rm NNC}\xspace}

\newcommand{\nmi}{{\ensuremath \rm{NMI}}\xspace}

\newcommand{\vlbs}{{\sc MKNN}\xspace}
\newcommand{\gridsearch}{{\sc GridSearch}\xspace}

\newcommand{\ourmethod}{{\sc OURS}\xspace}

\newcommand{\binarysearch}{{\sc BinarySearch}\xspace}

\newcommand{\winedata}{{\sc Wine}\xspace}
\newcommand{\auditdata}{{\sc Audit}\xspace}
\newcommand{\wifidata}{{\sc WiFi}\xspace}
\newcommand{\dermatologydata}{{\sc Derma.}\xspace}
\newcommand{\wdbcdata}{{\sc WDBC}\xspace}
\newcommand{\aggregationdata}{{\sc Aggr.}\xspace}
\newcommand{\jaindata}{{\sc Jain}\xspace}
\newcommand{\pathbaseddata}{{\sc Pathbased}\xspace}
\newcommand{\ddata}{{\sc D31}\xspace}
\newcommand{\rdata}{{\sc R15}\xspace}
\newcommand{\flamedata}{{\sc Flame}\xspace}
\newcommand{\spiraldata}{{\sc Spiral}\xspace}
\newcommand{\compounddata}{{\sc Compound}\xspace}
\newcommand{\mnistdata}{{\sc MNIST}\xspace}

\newcommand{\profileimgw}{.3\textwidth}
\newcommand{\timesimgw}{.3\textwidth}

\newcommand{\squishlist}{
 \begin{list}{$\bullet$}
  {  \setlength{\itemsep}{0pt}
     \setlength{\parsep}{3pt}
     \setlength{\topsep}{3pt}
     \setlength{\partopsep}{0pt}
     \setlength{\leftmargin}{2em}
     \setlength{\labelwidth}{1.5em}
     \setlength{\labelsep}{0.5em}
} }
\newcommand{\squishlisttight}{
 \begin{list}{$\bullet$}
  { \setlength{\itemsep}{0pt}
    \setlength{\parsep}{0pt}
    \setlength{\topsep}{0pt}
    \setlength{\partopsep}{0pt}
    \setlength{\leftmargin}{2em}
    \setlength{\labelwidth}{1.5em}
    \setlength{\labelsep}{0.5em}
} }

\newcommand{\squishdesc}{
 \begin{list}{}
  {  \setlength{\itemsep}{0pt}
     \setlength{\parsep}{3pt}
     \setlength{\topsep}{3pt}
     \setlength{\partopsep}{0pt}
     \setlength{\leftmargin}{1em}
     \setlength{\labelwidth}{1.5em}
     \setlength{\labelsep}{0.5em}
} }

\newcommand{\squishend}{
  \end{list}
}

\newcommand\numberthis{\addtocounter{equation}{1}\tag{\theequation}}

\begin{document}
\title{Off-the-grid: Fast and Effective Hyperparameter Search for Kernel Clustering\thanks{This work was supported by 
the Academy of Finland project 317085.}}
%
%

\toctitle{Off-the-grid: Fast and Effective Hyperparameter Search for Kernel Clustering}

\tocauthor{Bruno~Ordozgoiti and Llu\'is~Belanche~Mu\~noz}

\author{Bruno Ordozgoiti\inst{1} {\Letter} \and
Llu\'is A. Belanche Mu\~noz\inst{2}}

\authorrunning{B. Ordozgoiti\inst{1} \Letter \and
L. A. Belanche Mu\~noz}
%
\institute{Aalto University, Finland
\email{<firstname>.<lastname>@aalto.fi}\\
 \and
Universitat Polit\`ecnica de Catalunya, Spain\\
\email{belanche@cs.upc.edu} }
\maketitle              

\setcounter{footnote}{0}
\begin{abstract}
Kernel functions are a powerful tool to enhance the $k$-means
clustering algorithm via the kernel trick. It is known that the
parameters of the chosen kernel function can have a dramatic impact on
the result. In supervised settings, these can be tuned via
cross-validation, but for clustering this is not straightforward and
heuristics are usually employed. In this paper we study the impact of kernel parameters on kernel
$k$-means. In particular, we derive a lower bound, tight up to
constant factors, below which the parameter of the RBF kernel will
render kernel $k$-means meaningless. We argue that grid search can be
ineffective for hyperparameter search in this context and propose an
alternative algorithm for this purpose. In addition, we offer an
efficient implementation based on fast approximate exponentiation with provable
quality guarantees. Our experimental results demonstrate the ability
of our method to efficiently reveal a rich and useful set of hyperparameter values. 
\keywords{clustering  \and kernels \and kernel k-means \and hyperparameter tuning \and grid search.}
\end{abstract}
\section{Introduction}
Clustering, the task of partitioning a given data set into groups of
similar items, is one of the central topics in data analysis. Among the
plethora of existing techniques for this purpose, $k$-means
clustering, along with Lloyd's algorithm \cite{lloyd1982least}, is one
of the most popular and well-understood methods. Despite its
popularity, $k$-means has significant limitations, as it implicitly
makes strong assumptions about the shapes of the clusters. 
Numerous alternative methods have been proposed to tackle challenges beyond the capabilities of $k$-means
\cite{ester1996density,moon1996expectation,ng2002spectral,jain2010data}.
 
One of these involves the use of positive definite
kernels \cite{hofmann2008kernel}, which enable the computation of
inner products between elements of a vector space after mapping them
to a different, high-dimensional space. In particular, kernels enhance
the capabilities of $k$-means by enabling the detection of clusters
of arbitrary shapes.


One drawback of kernel functions is that they usually involve
hand-set parameters, which must be fine-tuned to bring forth their
full potential. A common method to choose a value for these parameters is grid search. One considers a set of values and then evaluates the
performance of the algorithm for each of them. A drawback is that one might either choose too small a set and risk missing optimal values, or an overly big one, incurring excessive ---and possibly redundant--- computational costs. Another way to set these values is by heuristics and rules of thumb \cite{von2007tutorial,jaakkola1999using}, but these rarely apply to a wide variety of data. 

Our contribution in this paper is two-fold. First, we illustrate the impact of kernel parameters in clustering by deriving a lower bound for the bandwidth parameter of RBF kernels (section \ref{sec:bound}), below which \kkm will be rendered
useless. We show this bound is tight. Next, we propose a method for hyperparameter search. Our method specifically searches for values that will produce different clusterings, and thus, unlike grid search, does not risk carrying out redundant computations, so no processing time is wasted. We combine methods for fast exponentiation with the properties of dyadic rationals to design an algorithm that after $\bigO \left ( \log\left ( \frac{|\log(b)|}{\epsilon} \right ) \right )$ iterations ---where $b$ is the minimum entry in the kernel matrix--- provides a $(1\pm\epsilon)$-approximation of the next meaningful hyperparmeter value to inspect (sections \ref{sec:search} and  \ref{sec:exponentiation}). We validate our claims with a rich variety of experiments (section \ref{sec:experiments}).

\section{Related work}
Kernels have been a central subfield of machine learning since their
first use in conjunction with support vector machines \cite{cortes1995support}. Even though most efforts have
focused on their application to supervised learning methods,
they have also played a significant part in the development of
clustering techniques \cite{ben2001support,ng2002spectral,dhillon2004kernel}.
%
%
In the seminal work by Ben-Hur et al. \cite{ben2001support}, the authors suggest to inspect
the results using varying values of $\sigma$, starting from the
maximizer of the pairwise squared distances $\|x-y\|^2$ over all pairs
of data points. A good choice might lie within a region that yields
stable clusterings. It should be noted that stability has been shown
to have significant drawbacks for choosing the number of
clusters \cite{ben2006sober}, so it would be interesting to determine
whether this applies to the kernel bandwidth as well. 
In the work that introduced spectral clustering \cite{ng2002spectral},
Ng et al. rely on a result of their own that guarantees that their algorithm
will produce tight clusters if they exist in the data. They then propose to test 
various values of $\sigma$ in search for a clustering with this property.
In \cite{bach2004learning} a generalized form of the bandwidth parameter is learned based
on data with known clustering.
In \cite{zelnik2005self} a different value of $\sigma$ is computed for
each point. The approach proposed by the authors relies on the distance to the $k$-th
neighbor. 
In \cite{greene2006practical}, the authors investigate the problem of
kernel matrix diagonal dominance in clustering, which is essentially a
generalization of the problem we analyze in the beginning of section \ref{sec:bound}. The heuristics
they explore to alleviate the problem either require the selection of a new
hyperparameter, or heavily modify the structure of the problem. The
latter can even lead to the loss of positive-definiteness of the
kernel matrix, which results in algorithmic oscillations and failure
to converge. 
The mean distance to the $k$-th nearest neighbour is also suggested as
a heuristic by Von Luxburg \cite{von2007tutorial}.
%

\section{Preliminaries}
We consider a finite set of data points $X\subset \mathbb R^d$. We define a
$k$-partition of $X$ as a collection of $k$ non-empty subsets of $X$, $\pi_1,
\dots, \pi_k$, satisfying $\bigcup_{i=1}^k\pi_i=X$ and
$\pi_i\cap\pi_j=\emptyset$ for $i,j=1,\dots,k$, $i\neq j$. We will
refer to each $\pi_i$ as a \textit{cluster} and use $n_i=|\pi_i|$ to
denote its cardinality.

The \textit{$k$-means} objective is to find a $k$-partition of
$X$ so as to minimize
\begin{equation}
  \label{kmeans}
\sum_{i=1}^k\sum_{x\in\pi_i}\|x-\bar\pi_i\|^2,
\end{equation}
where $\bar\pi_i=n_i^{-1}\sum_{x\in \pi_i}x$ is the
\textit{centroid} of cluster $\pi_i$ and $\|x\|$ denotes the
$L_2$ norm in $\mathbb R^d$.
Optimizing this objective is known to be \np-hard for $k=2$
\cite{aloise2009np}. A popular heuristic is Lloyd's
algorithm \cite{lloyd1982least}, which repeatedly recomputes the  
centroid of each cluster and reassigns points to the closest centroid.

 \subsubsection{Kernels:}
 Given a non-empty set $\mathcal{X}$, a symmetric function
 \(\kappa:\mathcal{X} \times \mathcal{X} \rightarrow \mathbb{R}\) such
 that for all \( n \in \mathbb{N}\) and every set
 \(\{x_i\}_{i=1}^n \subset \mathcal{X}\), the matrix
 $K=\left(\kappa(x_i,x_j)\right )_{ij}$ is positive definite, is called
 a (strictly) \textbf{positive definite (PD) kernel}. The matrix $K$ is
 known as the \textit{Gram} matrix or \textit{Kernel} matrix.
 Since PD kernels give rise to a PD Gram matrix,
 they correspond to the computation of an inner product in some
 implicit inner-product space. The representation of an object 
 $x\in \mathcal X$ in said space is often called \textit{feature space
 representation}, denoted as $\phi(x)$.

A number of kernels are routinely used in practice. Probably the most popular one for the case $\mathcal{X} =\mathbb{R}^d$ is
the Gaussian RBF kernel
\[
\kappa(x,y) = \exp\left(\frac{-\|x -y\|^2}{\sigma}\right),
\]
$\sigma>0$, (from here on, RBF kernel). The parameter $\sigma$
is commonly referred to as \textit{bandwidth}. We will use $\kappa_\sigma$ to denote the RBF kernel function with bandwidth parameter $\sigma$ and $K_\sigma$ to denote the corresponding kernel matrix.


Consider a data set $X$ and the $k$-partition
$\pi_1, \dots, \pi_k$. Let $m_i$ denote the centroid of cluster $\pi_i$
in feature space, that is,
\[
m_i=\frac{1}{n_i} \sum_{x\in\pi_i}\phi(x).
\]

The application of kernels to the $k$-means objective (\ref{kmeans}) relies on the
following observation: even though we cannot in general express $m_i$
explicitly, it is possible to compute the necessary squared distances. For any $x\in X$ and $i=1,\dots, k$,
\begin{align}
  \label{fs_distance}
    \|\phi(x)-m_i\|^2 =  \kappa(x,x) - \frac{2\sum_{y\in \pi_i}\kappa(x,y)}{n_i} + \frac{\sum_{y,z\in \pi_i}\kappa(y,z)}{n_i^2}.
\end{align}

 The application of Lloyd's algorithm using this expression for the
 squared distance is known as \kkm. See \cite{dhillon2004kernel} for
 an insightful analysis. \kkm always converges when the kernel matrix is positive semidefinite. We will refer to the $k$-partition at convergence as the \textit{output} of \kkm.

\section{The use of the RBF kernel in Kernel $k$-means}
\label{sec:bound}
RBF kernels are powerful but sensitive to the
bandwidth parameter. In particular, for sufficiently small $\sigma$, a
support vector machine classifier can fit any training set with no
errors ---or equivalently, it has infinite VC dimension
\cite{vapnik2006estimation}---, but this will generally result in
poor generalization ability. In \kkm, the result of an overly small
bandwidth will be that the algorithm will converge in the first
iteration, regardless of the current $k$-partition.
The reason is that as $\sigma$ decreases, the value
of $\kappa(x,y)$ for any two distinct points $x,y\in X$ decreases as well,
to the point of becoming negligible. Therefore, the only significant
term in equation (\ref{fs_distance}) for any $x$ will be $\kappa(x,x)$, which
means that the closest cluster to $x$ will be the one it is currently in. 
A question arises naturally: how small does
$\sigma$ have to be for the algorithm to get stuck at the initial
clustering? The following theorem provides a lower bound, which is
tight up to constant factors.

\smallskip
\begin{theorem}
  \label{the:bound}
Consider a data set $X \subset \mathbb R^d$, $|X|=n$. Let
$x,y=\argmin_{x,y \in X}\|x-y\|^2$. If $\sigma  \leq
\left(\log (3n)\right)^{-1}\|x-y\|^2$, then \kkm will make no
cluster reassignments.
\end{theorem}

The proof is given in the supplementary material.
%
\paragraph{A tight example.}
The next example shows that this result is tight up to constant factors.
Consider an instance with two clusters, $\pi_1$ and $\pi_2$,
containing $n_1$ and $n_2$ points respectively. For some point $y\in \pi_2$ it is
$\|x-y\|_2^2=\min_{a,b}\|a-b\|_2^2=\epsilon$ for all $x\in\pi_1$, whereas for all $z\in \pi_2, z\neq y$, it is
$\|y-z\|_2^2=2\epsilon$. Moreover, for all $w,z\in\pi_1$
it is $\|w-z\|_2^2=\epsilon$ and
for all $w,z\in\pi_2, w,z\neq y$,
it is $\|w-z\|_2^2=\epsilon$. Define $n=n_1+n_2$ and consider
$\sigma=\epsilon/\log(n/3)$. We know $y$ will switch over to
$\pi_1$ if $\|\phi(y)-m_1\|_2^2 < \|\phi(y)-m_2\|_2^2$, or
equivalently,
 \begin{multline}  
\frac{2}{n_2} <~  \frac{2\sum_{x\in \pi_1}\kappa(y,x)}{n_1} - \frac{\sum_{w,z\in \pi_1}\kappa(w,z)}{n_1^2} - \frac{2\sum_{z\in \pi_2, z\neq y}\kappa(y,z)}{n_2} + \frac{\sum_{w,z\in \pi_2}\kappa(w,z)}{n_2^2}
 \\ = 6/n - 1/n_1 - \frac{3(n_1-1)}{nn_1} - \frac{2(n_2-1)}{n_2}\left(\frac{3}{n}\right)^2 \\+ 1/n_2 + \frac{3(n_2-1)(n_2-2)}{nn_2^2} + \frac{(n_2-1)}{n_2^2}\left(\frac{3}{n}\right)^2.
\end{multline}
The above inequality is verified when $n_1=n_2$ and $n$ is sufficiently large. That is, there exists a
family of instances where the kernel $k$-means algorithm with the RBF
kernel will make cluster reassignments with $\sigma=\Omega\left(
\frac{\|x-y\|_2^2}{\log(n)}\right )$, where $\|x-y\|_2^2$ is minimal over all
$x,y$ in the data set.

\section{Optimizing bandwidth}
\label{sec:search}
As demonstrated above, the choice of bandwidth parameter is crucial when using RBF kernels for clustering. For some choices of $\sigma$, the output of \kkm will be unchanged from the initial $k$-partition. In fact, for any value of $\sigma$ the algorithm will converge at some point ---provided that the kernel matrix is positive semidefinite--- and stop making changes. However, if the chosen value is inadequate the output might still be of poor quality, so it is often desirable to further refine $\sigma$ in order to obtain a better result. We already know, by virtue of Theorem \ref{the:bound}, a value of $\sigma$ such that \kkm will stop making changes. The following question arises naturally. \textit{How big does $\sigma$ have to be in order to guarantee that \kkm will change the initial $k$-partition?}, and more generally, \textit{once \kkm has converged, how much do we have to increase $\sigma$ to ensure it will make new changes?} We define this as the \textit{critical bandwidth value}.

\begin{definition}{(Critical bandwidth value)}
  Let $X$ a data set. Suppose \kkm outputs a $k$-partition $P=(\pi_1, \dots, \pi_k)$ of $X$ when run using an RBF kernel with bandwidth parameter $\sigma$. We define $S \subset \mathbb R$ to be the set satisfying the following: if \kkm is initialized with $k$-partition $P$ and run with $K_{\sigma'}$, with $\sigma' \in S$, it will output a $k$-partition $P'\neq P$, that is, it will make changes. We define the critical bandwidth value with respect to $(K_\sigma,P)$ to be the infimum of $S$, or $\infty$ if $S=\emptyset$.
\end{definition}
In other words, the critical bandwidth value reveals the ``minimal'' value the RBF kernel bandwidth needs to take so that \kkm ``snaps out'' of convergence and yields a new $k$-partition. Any value strictly larger than the critical value will suffice. This concept is the cornerstone of our contribution.

\subsection{Finding the critical value}
Possibly the most straightforward method to find a value of $\sigma$ ---or virtually any hyperparameter--- is grid search. This consists in running the clustering algorithm for a predetermined set of values of the hyperparameter and choosing the one which provides the best performance, as measured by e.g. objective function values or clustering quality indices \cite{rousseeuw1987silhouettes}.
This approach, however, has significant disadvantages. If the set of values to test is too small, one can fail to detect one that yields good performance; if it is too large, running times can be prohibitive and some computations redundant.

Here we propose an alternative approach. Roughly, we proceed as follows. First, we choose a sufficiently small value of $\sigma$ ---e.g. guided by Theorem \ref{the:bound}--- and run \kkm. We then search for the critical bandwidth value with respect to the current kernel matrix and $k$-partition and rerun \kkm until convergence. We can keep doing this until no further changes are observed, to finally obtain a set of possible hyperparameter choices. The question that arises now is how to find said value efficiently. Next, we illustrate the fact that this value can be located using optimization methods.

\subsubsection{A first approach} Let $\kappa_\sigma$ denote the RBF kernel function parametrized by $\sigma$.
In a \kkm iteration, a point $x$ is assigned to the cluster $\pi_i$ which maximizes the \textit{proximity} function $\delta$:
\begin{align}
  \label{eq:proximity}
\delta_\sigma(x,m_i) = \frac{2\sum_{y\in \pi_i}\kappa_\sigma(x,y)}{n_i} - \frac{\sum_{y,z\in \pi_i}\kappa_\sigma(y,z)}{n_i^2}.
\end{align}

Now, observe that if we change the value of the bandwidth parameter to $\sigma'$, the new value of the kernel for any pair of points $x,y$ can be computed as follows:
\[
\kappa_{\sigma'}(x,y) = \kappa_{\sigma}(x,y)^{\sigma/{\sigma'}},
\]
and we can thus compute the new proximity functions $\delta_{\sigma'}(x,m_i)$ accordingly.
For simplicity, we consider the case of two clusters $\pi_1, \pi_2$. Assume $x \in \pi_1$. $x$ will switch over to $\pi_2$ when
\begin{align*}
   \delta_{\sigma'}(x,m_1) < \delta_{\sigma'}(x,m_2) \Leftrightarrow \delta_{\sigma'}(x,m_1) - \delta_{\sigma'}(x,m_2) < 0.
\end{align*}
That is, we can find the value of $\sigma'$ that will result in a different clustering by finding a root of $\delta_{\sigma'}(x,m_1) - \delta_{\sigma'}(x,m_2)$.

A useful observation is that $\kappa_\sigma(x,y)^\sigma$ is constant with respect to $\sigma'$. Therefore, we can easily derive $\delta_{\sigma'}(x,m_1) - \delta_{\sigma'}(x,m_2)$ with respect to $\sigma'$. In particular, define $g(x,\sigma') = \delta_{\sigma'}(x,m_1) - \delta_{\sigma'}(x,m_2)$. Then

\begin{multline}
  \frac{dg}{d\sigma'} = \frac{2\sum_{y\in \pi_2}\log\left(\kappa_\sigma(x,y)^{\sigma}\right)\kappa_\sigma(x,y)^{\sigma/{\sigma'}}}{\sigma'^2n_2} - \frac{\sum_{y,z\in \pi_2}\log\left(\kappa_\sigma(y,z)^{\sigma}\right)\kappa_\sigma(y,z)^{\sigma/{\sigma'}}}{\sigma'^2n_2^2}
  \\  -\frac{2\sum_{y\in \pi_1}\log\left(\kappa_\sigma(x,y)^{\sigma}\right)\kappa_\sigma(x,y)^{\sigma/{\sigma'}}}{\sigma'^2n_1} + \frac{\sum_{y,z\in \pi_1}\log\left(\kappa_\sigma(y,z)^{\sigma}\right)\kappa_\sigma(y,z)^{\sigma/{\sigma'}}}{\sigma'^2n_1^2}.
\end{multline}

This implies that we can use iterative root-finding algorithms, such as Newton's method, to efficiently find a root of the above function, that is, the minimum value of $\sigma'$ that will result in a clustering change, or the critical bandwith value. 

This approach, however, can be slow and numerically unstable. In the next section we propose an alternative optimization method able to efficiently locate the critical bandwidth value to arbitrary precision while overcoming these drawbacks.


\section{Fast and effective hyperparameter search}
\label{sec:exponentiation}
The approach outlined above has several drawbacks, namely (1) using an iterative root-finding algorithm entails repeatedly recomputing the kernel matrix, either directly or by element-wise exponentiation, which can be slow in practice when dealing with large matrices and (2) the operations required for the derivative of $g$ and the fractional computations can induce numerical instability.

Here we propose an alternative approach to sidestep these issues.
The proposed method rests on the following fact: \textit{computing products and square roots of real numbers can be much faster than computing powers with arbitrary exponents} \cite{gordon1998survey}. Our method has the additional advantage of being numerically stable.

\subsection{Dyadic rationals and fast approximate exponentiation}
To develop an efficient method for hyperparameter search, we first propose an algorithm for fast approximate exponentiation that only uses products and square roots. This algorithm (Algorithm \ref{alg:expo}) forms the basis of our approach.

\subsubsection{Exponentiation algorithm overview.} As hinted above, we wish to avoid computing element-wise powers of the kernel matrix, and instead use element-wise products and square roots. To accomplish this, suppose we want to compute the power $b^p$, for some arbitrary positive reals $b$ and $p$. We first decompose $p$ as $p=z+f$, where $z$ is the integral part and $f$ the decimal part of $p$. We then compute $b^z$ and approximate $b^f$ as $b^{f'}$ using two separate fast methods for integral and rational exponents and finally return $b^zb^{f'} \approx b^p$.

To design our algorithm, we rely on two simple results. First, we make use of the following recursive representation of a positive integer based on its binary representation, which has long been employed in the design of fast algorithms for power computation with integral exponents \cite{gordon1998survey}.
\begin{lemma}
  \label{lem:binaryrec}
  Consider a number $n \in \mathbb N$, and let $b_0\dots b_t$, where $t=\lfloor\log_2 n \rfloor$, be its binary representation, i.e. $n=\sum_{i=0}^t2^{t-i}b_i$.
  Then $n=n_t$, where
  \begin{align*}
    n_i=
    \begin{cases}
      1 &  \mbox{ if } i=0
      \\ 2n_{i-1} + b_i & \mbox{ if } 0<i\leq t
    \end{cases}
  \end{align*}
\end{lemma}
Lemma \ref{lem:binaryrec} reveals how to compute a power of the form $b^i$, where $b$ is a positive real number and $i$ is a natural number, using a small number of products. In particular, this operation is carried out in lines \ref{code:binpowerloop} and \ref{code:binpowerop} of Algorithm \ref{alg:expo}.

The next result we rely on is a consequence of the properties of dyadic rationals. Dyadic rationals are rational numbers of the form $n/2^i$, where $n$ is an integer and $i$ is a natural number. It is well known that dyadic rationals are dense in $\mathbb R$, that is, any real number can be approximated arbitrarily well by a dyadic rational.
The next result reveals how to obtain such an approximation for numbers in the interval $(0,1)$, which will be useful in our context.
\begin{lemma}
  \label{lem:power2approx}
Let $a\in (0,1)$. There exists a sequence $(m_i)$, with $m_i\in\{-1,1\}, i=1, \dots$ such that $\lim_{t\rightarrow \infty}\sum_{i=1}^{t}m_i2^{-i}=a$.
\end{lemma}
\begin{proof}
Let $m_1=1$. Choose the $j$-th term of $(m_i)$ (for $j>1$) to be $1$ if $\sum_{i=1}^{j-1}m_i2^{-i}<a$, $-1$ if $\sum_{i=1}^{j-1}m_i2^{-i}>a$, $0$ otherwise. Clearly, $\left | a-\sum_{i=1}^km_i2^{-i}\right| \leq 2^{-k}$.
\end{proof}
The set of dyadic rationals is clearly closed under addition, and thus the above series provides an approximation by means of a dyadic rational.

Now, suppose we want to approximately compute the power $b^p$, by an approximation of $p$ to within an error of $2^{-j}$. The above result implies that it suffices to compute $j$ operations, at each step either multiplying or dividing by successive square roots of $b$. This is done in lines \ref{code:ratpowerbegin} through \ref{code:ratpowerend} of Algorithm \ref{alg:expo}.
\begin{algorithm}[t]
  \caption{Fast approximate exponentiation}
  Input: base $b$, exponent $p$, depth $i$
  \begin{algorithmic}[1]
    \STATE $z \gets bin(\lfloor p \rfloor)[1:]$
    \STATE $f  \gets p - \lfloor p \rfloor$
    \STATE $b_1 \gets b$; $b_2 \gets b$
    \STATE $j\gets 1$
    \FOR {$d$ in $z$}  \label{code:binpowerloop}
    \STATE $b_1 \gets b_1^2b^d$  \label{code:binpowerop}
    \ENDFOR
    \STATE $n \gets 1$; $d \gets 2$
    \FOR {$j=1,\dots, i$} \label{code:ratpowerbegin}
    \STATE $b \gets \sqrt b$; $n \gets 2n$; $d \gets 2d$
    \IF {$n/d > f$}
    \STATE $n \gets n-1$; $b_2 \gets b_2/b$
    \ENDIF
    \IF{$n/d < f$}
    \STATE $n \gets n+1$; $b_2 \gets b_2b$
    \ENDIF
    \IF {$n/d = f$}
    \STATE $\j \gets i+1$ // Exact exponent matched, so exit loop \label{code:ratpowerend}
    \ENDIF
    \ENDFOR
    \STATE Output $b_1\times b_2$
  \end{algorithmic}
  \label{alg:expo}
\end{algorithm}
The following result characterizes the quality of the approximation achieved by Algorithm \ref{alg:expo}, and the required number of operations.
\begin{theorem}
\label{the:approxguarantee}
  Algorithm \ref{alg:expo} yields a $(1\pm\epsilon)-$approximation of $b^p$ after performing $\bigO \left ( \log\left ( \frac{|\log(b)|}{\epsilon} \right ) \right )$ operations.
\end{theorem}
\begin{proof}
  First, note that the algorithm computes at most $2i$ multiplications in the first phase, and $i$ square roots or multiplications in the second.
  
  Assume $b>1$. We treat the alternative later. By lemma \ref{lem:power2approx}, the output of Algorithm \ref{alg:expo} is bounded as follows
  \[
  \frac{b^p}{b^{1/2^i}}=b^{p-1/2^i} \leq r \leq b^{p+1/2^i}=b^pb^{1/2^i}.
  \]
  Observe that $b^pb^{1/2^i}=b^p+b^p(b^{1/2^i}-1)$ and set $\epsilon=b^{1/2^i}-1$. We thus have $\frac{1}{2^i}=\frac{\log(1+\epsilon)}{\log(b)}$ and thus $i=\bigO\left(\log\left(\frac{\log(b)}{\epsilon}\right)\right)$. Similarly, we can write $b^pb^{-1/2^i}=b^p-b^p(1-b^{-1/2^i})$, arriving at an equivalent result for the $1-\epsilon$ bound.

  The analysis for the case $b<1$ is the same, but noting that the output is bounded as $b^{p+1/2^i}\leq r \leq b^{p-1/2^i}$. The negative sign of $\log(b)$ is cancelled out in the arithmetic. The case $b=1$ is obviously of no interest.
\qed
\end{proof}

%
%

Algorithm \ref{alg:expo} approximates a power computation by a dyadic rational approximation $w/z$ of the exponent.
Based on the principles behind Algorithm \ref{alg:expo} we can design an efficient method to find the critical value of $\sigma$ for \kkm.
\subsubsection{Finding the critical value.} Our algorithm for hyperparameter search is detailed as Algorithm \ref{alg:search}. In the pseudocode, $\circ$ and $/\circ$ denote element-wise multiplication and division, respectively, and $\sqrt K$ is the element-wise square root of matrix $K$.

In essence, our algorithm emulates Algorithm \ref{alg:expo}, using the kernel matrix $K_\sigma$ as the basis of the power to compute, with some key differences. The first difference is that instead of approximating a known exponent $p$, we aim to approximate the \textit{unknown} critical value of $\sigma$. Since this quantity is unknown, instead of testing whether the current approximation is larger or smaller than the target exponent, we query the \kkm algorithm to determine whether the current value will result in new changes. Note that this amounts to running a single iteration of \kkm. Later we show that we can further optimize these queries.

The second observation is that we only ever need to compute exponents in the interval $(0,1)$. This is because if we assume \kkm to have converged for the matrix $K_\sigma$, we know that the next value of $\sigma$ we seek is larger than the current one. Note that we can use our result from Theorem \ref{the:bound} for a starting value of $\sigma$ without running an initial execution of \kkm.

%
By virtue of Theorem \ref{the:approxguarantee}, Algorithm \ref{alg:search} thus finds an arbitrarily good approximation of the critical bandwidth value, in the following sense:
\begin{corollary}
  Suppose \kkm has converged for $K_\sigma$, producing a $k$-partition $P$, and let $\sigma'$ be the critical bandwidth value with respect to $(K_\sigma, P)$. If we run Algorithm \ref{alg:search} with a depth value of $i=\bigO \left ( \log\left ( \frac{|\log(b)|}{\epsilon} \right ) \right )$ ---where $b$ is the minimum entry in the kernel matrix---, it will output a matrix $K_\rho$ satisfying
  \[
  (1-\epsilon)K_{\sigma'} \leq_\circ K_\rho \leq_\circ (1+\epsilon)K_{\sigma'},
  \]
  where $\leq_\circ$ denotes element-wise inequality.
\end{corollary}
That is, it will output a good approximation of the ``next'' kernel matrix for which \kkm will make changes. Note that this result also characterizes the computational complexity of our approach, as element-wise operations take $\bigO(n^2)$ computations. In addition, element-wise operations are trivially parallelizable, so our method can scale to large kernel matrices. Finally, note that even though $\log(b)$ is unbounded, after a few iterations only very small entries, close to zero, would suffer considerable relative error.

An advantage of the algorithm is that we can choose the maximum value of the denominator in the rational approximation of the exponent (maximum depth $d$). This provides a nice trade-off between speed and accuracy.

\begin{algorithm}[t]
  \caption{Hyperparameter search}
  Input: kernel matrix $K$, depth $i$, $k$-partition $P$ of $X$.
  \begin{algorithmic}[1]
    \STATE $K' \gets K$
    \STATE $P' \gets P$
    \STATE $j\gets 1$
    \FOR {$i=1, \dots, i$}
    \STATE $K' \gets \sqrt {K'}$ //Element-wise square root
    \IF {$P' \neq P$} \label{algline:change}
    \STATE $K \gets K /\circ K'$
    \ELSE 
    \STATE $K \gets K \circ K'$
    \ENDIF
    \STATE $P' \gets kkm(K)$ //Run \kkm 
    \ENDFOR
    \STATE Output $K$
  \end{algorithmic}
  \label{alg:search}
\end{algorithm}

\subsection{Further optimizations}
Our approach lends itself naturally to various optimizations. We discuss them briefly here.

\paragraph{Hierarchical search.} Our algorithm enables a trade-off between running time and accuracy by means of the depth parameter. The larger it is, the more precise the critical values of $\sigma$ found. We argue that this parameter can be employed to improve speed without significantly sacrificing accuracy. In particular, the algorithm can be run with increasing depth values, constraining the search to promising regions. For instance, we first set depth to 1, run the algorithm and pick the two values of $\sigma$ that yield the best performance. We then increase the depth value by 1 and run the algorithm again, setting the lower and upper limits of our search to the two previously picked values of $\sigma$. This way we first perform a coarse-grained search to identify a potentially good interval for $\sigma$, and then increasingly refine the search.

\paragraph{Limiting checks.} As described above, the way our algorithm approximates the critical value of $\sigma$ is by testing whether or not \kkm will switch at least one point from one cluster to another. Often, most points will not switch clusters at the critical value. Thus, it is not necessary to compute the proximity function (Equation (\ref{eq:proximity})) for all point-cluster pairs, and we can limit checks to those points most likely to change. To do this, we can employ different heuristics. For instance, we can limit checks to points such that the proximity function is close for different clusters. We can also limit checks to those points that switch clusters the first time we observe a change (line \ref{algline:change} of Algorithm \ref{alg:search}).

\subsection{Use with other kernels}
Our approach is not limited to the RBF kernel. Obviously, any kernel that is exponential in the parameters can be directly used with our method. This includes the popular polynomial kernel, defined as $\kappa(x,y) = (x^Ty+c)^d$, for the optimization of the parameter $d$. We can also benefit from the fact that any linear combination of kernels is also a kernel, to accommodate a wider variety of kernel functions. To use our algorithm with a linear combination of differently-parametrized kernels, it suffices to store the kernel matrix separately for each term of the sum. As currently described, our method only allows the optimization of one parameter at a time, but it can be employed as a building block for more sophisticated multiparameter optimization approaches.

\section{Experiments}
\label{sec:experiments}
We conduct a series of numerical experiments to evaluate the performance of the proposed algorithm. We mainly want to determine whether our method (1) can reveal good value of $\sigma$ and (2) can do it efficiently.
We compare it to other approaches for hyperparameter search, which we now describe.
\paragraph{Baselines}
We consider the following methods to choose the hyperparameter of the RBF kernel\footnote{Some of these methods, as originally described, define the kernel as $\kappa(x,y)=\exp(-\|x-y\|/(2\sigma^2))$. We take this difference into account in our experimental setup.}.

\vlbs: We set $\sigma$ to be the mean distance to the $k$-th nearest neighbour as suggested by Von Luxburg \cite{von2007tutorial} (the median yields similar results). We try different values of $k$, namely $k=1, \dots 2(\log n +1)$.

\gridsearch: We run the \kkm algorithm with $\sigma$ taking values in $\{10^i: i=-6,-5,-4,-3,-2,-1,0,1,2,3,4,5,6\}$

We refer to Algorithm \ref{alg:search} as \ourmethod.

All methods, as well as \kkm, were implemented using Python 3, using matrix and vector operations whenever possible for efficiency\footnote{Source code: \url{https://github.com/justbruno/off-the-grid/}}.

\paragraph{Quality measures:} We consider the following functions to evaluate the quality of the clustering results.

\nmi (Normalized Mutual Information): We use a well-known clustering performance index\footnote{Results for Adjusted Rand-Index were similar and are thus omitted.}, which we now define. Given two indicator vectors $y$ and $z$, we define
\begin{align}
\label{eq:nmi}
\nmi(y,z) = \frac{2I(y,z)}{H(y)+H(z)}
\end{align}
where  $I(y,z)=\sum_i\sum_jp(y=i, z=j)\log\left( \frac{p(y=i, z=i)}{p(y=i)p(z=i)} \right)$ denotes the mutual information of $y$ and $z$,
and $H(y)=-\sum_ip(y=i)\log p(y=i)$ denotes the entropy of $y$ \cite{cover2012elements} (we abuse notation and overload $y$ for the vector and its entries). We use
this index by taking $y$ to be the indicator vector of
ground-truth labels and $z$ to be the indicator vector of the $k$-partition output by
\kkm.

$\nncost$: In addition, we propose our own clustering cost function. Our goal is to measure the quality of the resulting $k$-partition in a way that (1) arbitrarily shaped clusters are considered and (2) is independent of the value of $\sigma$. Note that some well-known clustering quality indices and cost functions, such as silhouette \cite{rousseeuw1987silhouettes} and normalized cuts \cite{dhillon2004kernel}, do not qualify.

We first introduce some notation. Given a data set $X$ and a point $x_i\in X$, let $\nu_j(x_i)$ be the $j$-th nearest neighbour of $x_i$ in $X$. Given a $k$-partition of the data set $X$ into $k$ clusters, $c(x_i)$ denotes the cluster $x_i$ is assigned to, i.e. $x_i\in c(x_i)$. 

We first define $\nnck(i,\nnk)$ to be the fraction of points among the $\nnk$ nearest neighbours of $x_i$ which are not in the same cluster as $x_i$.
\begin{align*}
\nnck(i,\nnk) = \frac{1}{\nnk}\sum_{j=1}^\nnk\mathbb I\{c(x_i) \neq c(\nu_j(x_i))\}.
\end{align*}

To measure the quality of a single cluster $\pi$, we take a weighted sum of the above index for all $\nnk$. We scale the value of $\nnck(i,\nnk)$ by $\frac{1}{\nnk}$ to reduce the penalty incurred by disagreements with further neighbours. 
\begin{align*}
\nnck_{cluster}(\pi) = \frac{1}{C\max\{1,|\pi|)\}}\sum_{i\in \pi}\sum_{\nnk=1}^n\frac{1}{\nnk}\nnck(i,\nnk).
\end{align*}
Here, $C=\log(n-1) + \gamma + \frac{1}{2n-2}$, where $\gamma$ is the Euler-Mascheroni constant, ensures that the quantity is upper-bounded by 1 (note that without this scaling factor, the sum for each point is tightly upper bounded by a harmonic series).

We now define the cost function as
\begin{align*}
  \nncost(P) = \frac{D+\sum_{\pi\in P}\nnck_{cluster}(\pi)}{k}.
\end{align*}
Here, $P$ is the $k$-partition output by \kkm, $k$ is the number of clusters given to \kkm and $D$ is the number of empty clusters. We count empty clusters to penalize trivial solutions (e.g. a single cluster).

\paragraph{Datasets}: We employ a variety of publicly available synthetic\footnote{\url{http://cs.joensuu.fi/sipu/datasets}} and real\footnote{\url{https://archive.ics.uci.edu/ml/index.php}} data sets. Since we use vanilla \kkm, which requires handling the complete kernel matrix, we employ data sets of limited size (up to $8\,000$ instances). However, our method can in principle be employed with techniques for scalable kernel-based algorithms. A summary of the data sets is given in Table \ref{tab:datasets}. In the case of real data sets, we scale the variables to unit-variance, as this enables a much better performance of \kkm in most cases.

\begin{table}
  \caption{Summary of data set characteristics}
  \label{tab:datasets}
  \begin{tabular}{c|r|r|r||c|r|r|r|}
    Dataset & Rows & Columns & Classes & Dataset & Rows & Columns & Classes  \\ \hline  \hline
    \aggregationdata & 788 & 2 & 7 &     \spiraldata & 312 & 2 & 3 \\ \hline
    \compounddata & 399 & 2 & 6 &     \auditdata & 775 & 23 & 2 \\ \hline
    \ddata & 3100 & 2 & 31 &     \dermatologydata & 358 & 34 & 6 \\ \hline
    \flamedata & 240 & 2 & 2 &     \wdbcdata & 569 & 30 & 2 \\ \hline
    \jaindata & 373 & 2 & 2 &     \wifidata & 2000 & 7 & 4 \\ \hline
    \pathbaseddata & 300 & 2 & 3 &     \winedata & 178 & 13 & 3 \\ \hline
    \rdata & 600 & 2 & 15 & \mnistdata (sampled) & 1k,2k,4k,8k&  784 & 10 \\ \hline
  \end{tabular}
\end{table}

\subsection{Performance}
In this section we report the performance of our method, as evaluated by our quality measures, in comparison to the selected baselines. We proceed as follows: we first choose a random initial $k$-partition, which we set as starting point for all methods.
To evaluate our method, we set the initial value of $\sigma$ to be the 1st percentile of pairwise distances in the data set. Note this is similar to our lower bound given in section \ref{sec:bound}, but a little less stringent. We run Algorithm \ref{alg:search} with depth=1 and pick the value of $\sigma$ that corresponds to the best observed $k$-partition (as measured by $\nncost$), run \kkm and rerun our method starting from the resulting $k$-partition with depth$=2$. Note that this resembles the hierarchical search described in section \ref{sec:exponentiation}.
For each method, we collect the best value of \nmi and $\nncost$ among the produced clusterings. We report the average over 50 runs, each with a different initial $k$-partition. Results are shown in Table \ref{tab:performance}. Our method achieves better values of both measures in most cases.
%

\begin{table}
  \caption{Comparison of the different methods in terms of quality measures}
  \label{tab:performance}
  \centering
  \begin{tabular}{c|r|r|r|r|r|r|}
    & \multicolumn{3}{c|}{\nmi} & \multicolumn{3}{c}{$\nncost$} \\ \hline
    Dataset & \vlbs & \gridsearch & \ourmethod & \vlbs & \gridsearch & \ourmethod  \\ \hline  \hline
    \aggregationdata & 0.690 & 0.864 & \textbf{0.872} & 0.255 & 0.210 & \textbf{0.203} \\ \hline
    \compounddata & 0.689 & \textbf{0.778} & 0.730 & 0.239 & 0.230 & \textbf{0.215} \\ \hline
    \ddata & 0.810 & 0.931 & \textbf{0.951} & 0.356 & 0.332 & \textbf{0.316} \\ \hline
    \flamedata & 0.489 & 0.521 & \textbf{0.615} & 0.106 & 0.096 & \textbf{0.093} \\ \hline
    \jaindata & 0.229 & \textbf{0.361} & 0.353 & 0.116 & 0.062 & \textbf{0.062} \\ \hline
    \pathbaseddata & 0.820 & 0.662 & \textbf{0.902} & 0.169 & \textbf{0.134} & 0.137 \\ \hline
    \rdata & 0.922 & 0.954 & \textbf{0.979} & 0.302 & 0.300 & \textbf{0.274} \\ \hline
    \spiraldata & 0.187 & 0.145 & \textbf{0.239} & 0.175 & 0.155 & \textbf{0.151} \\ \hline
    \auditdata & \textbf{0.717} & 0.685 & 0.703 & 0.097 & \textbf{0.082} & \textbf{0.082} \\ \hline
    \dermatologydata & 0.889 & 0.877 & \textbf{0.913} & 0.249 & 0.256 & \textbf{0.238} \\ \hline
    \wdbcdata & 0.531 & 0.547 & \textbf{0.550} & 0.123 & 0.108 & \textbf{0.107} \\ \hline
    \wifidata & 0.781 & 0.835 & \textbf{0.856} & 0.157 & 0.140 & \textbf{0.137} \\ \hline
    \winedata & \textbf{0.923} & 0.913 & \textbf{0.923} & 0.143 & \textbf{0.142} & 0.143 \\ \hline
  \end{tabular}
\end{table}

\subsection{Running times and Scalability}
In this section we evaluate the efficiency of our method. We report the average total running times in the previously described experiment for all algorithms in Table \ref{tab:times}. Our method generally sits between \gridsearch and \vlbs. It performs significantly more iterations than the baselines, and thus better running times could be obtained by limiting the number of inspected values if necessary.

\begin{table}
  \centering
  \caption{Total running times in seconds}
  \label{tab:times}
  \begin{tabular}{c|r|r|r||c|r|r|r|}
    & \multicolumn{3}{c|}{Time in seconds} & & \multicolumn{3}{c|}{Time in seconds}  \\ \hline
    Dataset & \vlbs & \gridsearch & \ourmethod & Dataset & \vlbs & \gridsearch & \ourmethod  \\ \hline  \hline
    \aggregationdata & 0.824 & \textbf{0.486} & 0.617 &    \spiraldata & 0.157 & \textbf{0.108} & 0.133 \\ \hline
    \compounddata & 0.200 & \textbf{0.146} & 0.172 &    \auditdata & 0.578 & \textbf{0.467} & 1.111 \\ \hline
    \ddata & 22.029 & 11.735 & \textbf{10.757} &    \dermatologydata & 0.152 & 0.106 & \textbf{0.100} \\ \hline
    \flamedata & 0.064 & \textbf{0.046} & 0.061 &    \wdbcdata & 0.325 & \textbf{0.204} & 0.213 \\ \hline
    \jaindata & 0.140 & \textbf{0.097} & 0.143 &    \wifidata & 7.095 & \textbf{3.520} & 4.402 \\ \hline
    \pathbaseddata & 0.110 & \textbf{0.077} & 0.094 &    \winedata & 0.044 & \textbf{0.034} & 0.036 \\ \hline
    \rdata & 0.475 & \textbf{0.329} & 0.406 & \multicolumn{4}{c|}{} \\ \hline
  \end{tabular}
\end{table}
To offer a finer running time comparison, as well as to evaluate scalability, we run the algorithms on samples of \mnistdata\footnote{\url{http://yann.lecun.com/exdb/mnist/}} and set the number of iterations to be the same for all methods. In particular, we set it to 13, which is the number of values tested by \gridsearch. Figure \ref{fig:time_per_iteration} shows time taken per iteration, averaged over 50 runs. By iteration we refer to the set of computations required to produce and test a new value of the bandwidth parameter. The reason the running time of \gridsearch increases significantly at some point is that the first values of $\sigma$ are too small and \kkm converges after one iteration, highlighting the wasteful nature of \gridsearch. Our method benefits mostly from being able to run a small number of iterations of \kkm to converge.
\begin{figure}
  \centering
  \includegraphics[width=\textwidth]{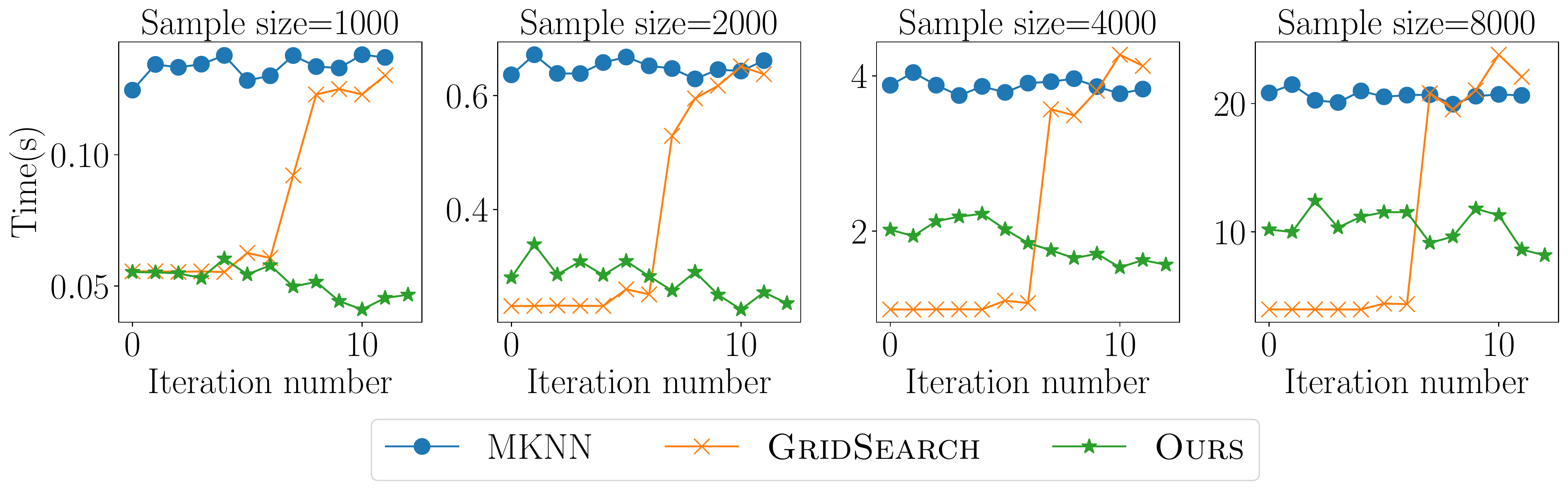}
  \caption{Running time per iteration for different samples of the \mnistdata data set}
  \label{fig:time_per_iteration}
\end{figure}

\subsection{Comparison with binary search}
The reader might observe that our method resembles a form of binary search. Thus, one might suspect that similar results could be obtained using a conventional binary search algorithm, without going to the trouble of implementing Algorithm \ref{alg:search}. Here we illustrate why our algorithm is a vastly superior alternative.

The setup is as follows: we initialize $\sigma$ to be the 1st percentile of the squared pairwise distances and then run iterations of binary search with a precision of $10^{-3}$ and Algorithm \ref{alg:search} with depth equal to $10$. We repeat the experiment $10$ times and report average iteration time and absolute error of the estimate of the critical value of $\sigma$. The results are shown in Table \ref{tab:search_comp}. Binary search was implemented efficiently, updating the kernel matrix with fast matrix-vector operations.

Our method achieves a speedup of about 10x in all cases, and the error is often smaller. Of course, the error can be controlled in both algorithms at the expense of running time. A noteworthy difference between both methods (not in favor of any of the two) is that binary search is designed to control absolute error, while Algorithm \ref{alg:search} controls the relative error of the power computation.
\begin{table}
  \centering
  \caption{Running times of our method and binary search. We report average iteration running times, speedup and mean relative error of the $\sigma$ estimate over 100 iterations}
  \label{tab:search_comp}
  \begin{tabular}{c|c|c|c|c|c|}
    & \multicolumn{2}{c|}{Iteration time in seconds} &Speedup & \multicolumn{2}{c|}{Relative error: $\frac{\sigma_{\rm{true}}-\sigma_{\rm{estimated}}}{\sigma_{\rm{true}}}$}  \\ \hline
    Dataset & \binarysearch & \ourmethod & - & \binarysearch & \ourmethod  \\ \hline  \hline
\aggregationdata & 0.941 & 0.080 &11.7x & $1.55\times 10^{-3}$  & $5.3\times 10^{-4}$ \\ \hline
\auditdata & 0.793 & 0.069 &11.5x & $8.341\times 10^{-2}$  & $5.8\times 10^{-4}$ \\ \hline
\compounddata & 0.192 & 0.019 &9.9x & $2.20\times 10^{-3}$  & $5.8\times 10^{-4}$ \\ \hline
\ddata & 15.740 & 1.148 &13.7x & $3.95\times 10^{-3}$  & $4.8\times 10^{-}$ \\ \hline
\dermatologydata & 0.139 & 0.014 &9.7x & $1.3\times 10^{-4}$  & $8.368\times 10^{-2}$ \\ \hline
\flamedata & 0.063 & 0.007 &9.1x & $5.1\times 10^{-3}$  & $5.6\times 10^{-4}$ \\ \hline
\jaindata & 0.144 & 0.014 &10.5x & $2.6\times 10^{-3}$  & $1.17\times 10^{-2}$ \\ \hline
\pathbaseddata & 0.096 & 0.010 &9.7x & $1.98\times 10^{-3}$  & $5.9\times 10^{-4}$ \\ \hline
\rdata & 0.430 & 0.039 &11.0x & $4.912\times 10^{-2}$  & $5.5\times 10^{-4}$ \\ \hline
\spiraldata & 0.102 & 0.011 &9.6x & $1.13\times 10^{-3}$  & $6.2\times 10^{-4}$ \\ \hline
\wdbcdata & 0.398 & 0.036 &11.2x & $10^{-6}$  & $6.2\times 10^{-4}$ \\ \hline
\wifidata & 5.284 & 0.442 &11.9x & $6\times 10^{-5}$  & $4.9\times 10^{-4}$ \\ \hline
\winedata & 0.042 & 0.005 &7.7x & $2\times 10^{-5}$  & $5.7\times 10^{-4}$ \\ \hline
  \end{tabular}
\end{table}

\section{Conclusion}
In this paper we have addressed the problem of hyperparameter search in the \kkm context. Our contribution is two-fold. First, we have derived a tight lower bound for the bandwidth parameter of RBF kernels, below which \kkm will be rendered useless. Second, we have proposed a method to optimize kernel hyperparameters for \kkm. We have proved that our method approximates critical values of the hyperparameter to arbitary precision in a small number of iterations. Unlike grid search or other heuristics, our method does not test redundant hyperparameter values, that is, values that result in the same clustering output, and thus no computation is wasted.

Our experiments demonstrate how our approach enables the efficient evaluation of a fine variety of hyperparameter values, revealing high-quality clustering results at a moderate computational cost. In the future it would be interesting to extend our method to other kernel-based clustering and classification algorithms.

 \bibliographystyle{splncs04}

%

\section*{Appendix}
\begin{proofv} Consider the squared
  distance as written in Eq. (\ref{fs_distance}). For the choice of cluster,
  we can drop the constant term $\kappa(x,x)$ and compute  
  \begin{align}
    \label{eq:argmin}    
    \argmin_j ~ - \frac{2\sum_{y\in \pi_j}\kappa(x,y)}{n_j} + \frac{\sum_{y,z\in \pi_j}\kappa(y,z)}{n_j^2}.
  \end{align}
  If $x\in \pi_j$ we can write  
  \begin{align}
    \label{eq:argmin_in}
    & - \frac{2\sum_{y\in \pi_j}\kappa(x,y)}{n_j} + \frac{\sum_{y,z\in \pi_j}\kappa(y,z)}{n_j^2} \\
    =& - \frac{2\kappa(x,x)}{n_j} - \frac{2\sum_{y\in \pi_j, y\neq x}\kappa(x,y)}{n_j} + \frac{\sum_{y,z\in \pi_j}\kappa(y,z)}{n_j^2}.
  \end{align}
If $x \in \pi_j$ and $\|\phi(x)-m_j\|_2^2\leq \|\phi(x)-m_i\|_2^2$
for all $i\neq j$, then $a$ will remain in the same
cluster. Considering (\ref{eq:argmin}) and (\ref{eq:argmin_in}), we can write this condition as
 \begin{align*}  
 \frac{2\kappa(x,x)}{n_j} \geq & \frac{2\sum_{y\in \pi_i}\kappa(x,y)}{n_i} - \frac{\sum_{y,z\in \pi_i}\kappa(y,z)}{n_i^2} \\
& - \frac{2\sum_{y\in \pi_j, y\neq x}\kappa(x,y)}{n_j} + \frac{\sum_{y,z\in \pi_j}\kappa(y,z)}{n_j^2}.
\end{align*}

Since $\kappa(x,y) \geq 0$ for any pair of points $x,y$, we can drop
the negative terms on the right-hand side to obtain the following,
more restrictive, condition: 
\begin{align*}
  &\frac{2\kappa(x,x)}{n_j} \geq  \frac{2\sum_{y\in \pi_i}\kappa(x,y)}{n_i} + \frac{\sum_{y,z\in \pi_j,y\neq z}\kappa(y,z)+n_j}{n_j^2} \\
  &= \frac{2n_j^2\sum_{y\in \pi_i}\kappa(x,y) + n_i\left ( \sum_{y,z\in \pi_j,y\neq z}\kappa(y,z) + n_j\right )}{n_in^2_j}.
  \numberthis\label{eq:ineq_more_rest}
\end{align*}
Here we have used $\sum_{x\in \pi_j}\kappa(x,x)=n_j$.
If we define $\omega = \max_{x\neq y}\kappa(x,y)$, then the two following inequalities hold:
\begin{align*}  
  n_i\omega \geq \sum_{y\in \pi_i}\kappa(x,y), &&
   n_j^2\omega \geq \sum_{y,z \in \pi_j,y\neq z}\kappa(y,z)
\end{align*}  
We can thus consider the following, more restrictive, condition (recall that $\kappa(x,x)=1$):
\begin{align}
  \label{eq:last_ineq_nj}
  \frac{2\kappa(x,x)}{n_j} & \geq \frac{2n_j^2n_i\omega + n_i\left
    (n_j^2\omega + n_j\right )}{n_in^2_j} = 3\omega + \frac{1}{n_j} \Leftrightarrow \frac{1}{n_j}  \geq 3\omega.
\end{align}

Trivially, $\frac{1}{n}  \geq 3\omega \Rightarrow
\frac{1}{n_j}  \geq 3\omega$. Now, after minor computational efforts,
it is
\begin{align}
  \label{eq:last_ineq} 
  \frac{1}{n}  \geq 3\omega   
  \Leftrightarrow -\log n  \geq \log 3 - \frac{\|x-y\|_2^2}{\sigma} \Leftrightarrow \frac{\|x-y\|_2^2}{\log 3n} \geq \sigma,
\end{align}
where $x,y=\argmin_{x.y\in X}\|x-y\|_2^2$.
Therefore, we have (\ref{eq:last_ineq}) $\Rightarrow$
(\ref{eq:last_ineq_nj}) $\Rightarrow$ (\ref{eq:ineq_more_rest})
$\Rightarrow$ $\|\phi(x)-m_j\|_2^2\leq \|\phi(x)-m_i\|_2^2$ for all
$i\neq j$. Since $x$ is an arbitrary element of $X$, (\ref{eq:last_ineq}) is 
a sufficient condition for kernel $k$-means to make no changes. 
%
\qed
\end{proofv}\

\end{document}